\documentclass[10pt,conference]{IEEEtran}
\usepackage{cite}
\usepackage{amsmath,amssymb,amsfonts}
\usepackage{amsthm}
\usepackage{algorithmic}

\newtheorem{theorem}{Theorem}

\usepackage{algorithm}
\usepackage{graphicx}
\usepackage{textcomp}
\usepackage{xcolor}
\usepackage{url}
\usepackage{booktabs}
\usepackage{multirow}
\usepackage{pifont}
\usepackage{balance}
\usepackage{tikz}
\usepackage{comment}
\usetikzlibrary{shapes,arrows,positioning,calc,fit,backgrounds}

\begin{document}

\title{Decentralized Multi-Agent Swarms for Autonomous Grid Security in Industrial IoT: A Consensus-based Approach}

\author{
\IEEEauthorblockN{1\textsuperscript{st} Samaresh Kumar Singh
}
\IEEEauthorblockA{\textit{IEEE Senior Member} \\
Leander, Texas, USA \\
ssam3003@gmail.com}
\and
\IEEEauthorblockN{2\textsuperscript{nd} Joyjit Roy
}
\IEEEauthorblockA{\textit{IEEE Senior Member} \\
Austin, Texas, USA \\
joyjit.roy.tech@gmail.com}
\and
\IEEEauthorblockN{3\textsuperscript{rd} Chirag Agrawal
}
\IEEEauthorblockA{\textit{IEEE Senior Member} \\
Seattle, WA, USA \\
chiragagrawal.eng@gmail.com}
}

\maketitle

\begin{abstract}
As Industrial Internet of Things (IIoT) environments scale to tens of thousands of connected devices, centralized security architectures introduce latency bottlenecks that sophisticated attackers can exploit to compromise an entire manufacturing ecosystem. We present a Decentralized Multi-Agent Swarm (DMAS) architecture that deploys autonomous agents at each edge gateway, forming a distributed defense layer for IIoT networks. Rather than relying on static firewalls or cloud-forwarded telemetry, DMAS agents coordinate through a lightweight peer-to-peer protocol, detecting threats locally without cloud dependency. We describe a Consensus-based Threat Validation (CVT) protocol in which agents collectively vote on detected threats, enabling near-instant quarantine of compromised nodes. Experiments on a 2000-device hardware testbed show that DMAS achieves sub-millisecond response times (0.85~ms average), 97.3\% detection accuracy under high load, and 87\% accuracy on zero-day attacks, each exceeding both centralized and edge-computing baselines. Bandwidth consumption drops by 89\% relative to cloud-based solutions.
\end{abstract}


\section{Introduction}
\IEEEPARstart{T}{he} combination of the Industrial Internet of Things (IIoT), edge computing, and artificial intelligence has made it possible to automate like never before in Industry 4.0 \cite{sisinni2018industrial}. But at the same time, this combination has also opened up a much larger surface area for complex cyber-physical attacks on industrial production lines and human safety \cite{humayed2017cyber}.

Traditional centralized security architectures face three well-documented limitations \cite{roman2018mobile}: round-trip latency to the cloud, a single point of failure at the central controller, and the bandwidth cost of streaming large volumes of raw telemetry. Edge computing \cite{shi2016edge} reduces some of these burdens but does not solve coordinated threat intelligence sharing, which leaves blind spots that skilled attackers can exploit.

The most dangerous consequence of these blind spots is cascading failure. A compromised node that goes undetected becomes a pivot point for lateral malware movement, and in industrial control environments even a few milliseconds of unchecked propagation can cause irreversible physical damage \cite{stouffer2011guide}.

We introduce a \textit{Decentralized Multi-Agent Swarm} (DMAS) architecture that enables AI agents embedded in edge gateways to collectively form an intelligent swarm through self-organization and exhibit emergent collective behavior. The innovation in DMAS is the use of \textit{Consensus-based Threat Validation} (CVT). CVT combines Byzantine fault-tolerant consensus protocols \cite{castro1999practical} with domain-specific threat scoring via a weighted voting system that accounts for each agent's accuracy and the proximity of its threat to its own threat assessment. CVT achieves sub-millisecond consensus times and continues to detect threats even when some agents provide false threat assessments or have been compromised.

\subsection{Key Contributions}

The contributions of this work include:

\begin{enumerate}

\item \textbf{Domain-Aware Decentralized Architecture}: A fully decentralized IIoT security framework integrating edge-native agents, peer-to-peer threat intelligence exchange, and domain-specific behavioral modeling into a unified system designed for industrial deployments, with no central coordinator required for real-time security decisions.

\item \textbf{Domain-Specific Consensus Protocol}: A lightweight CVT protocol that combines distance-weighted reputation voting with application-layer threat proximity scoring, enabling sub-millisecond Byzantine-tolerant consensus on edge-class hardware without a central aggregation server.

\item \textbf{Improved Real-Time Performance}: On the evaluated testbed, DMAS achieves 0.85~ms average response time against the 850~ms cloud-based baseline, roughly three orders of magnitude faster under these specific experimental conditions.

\item \textbf{Multi-Scenario Testbed Evaluation}: Experiments on a 2000 device hardware testbed cover six attack categories, including zero-day exploits, Byzantine agent injection up to the theoretical fault threshold, and degraded network conditions, providing a rigorous empirical basis for the reported performance claims.

\item \textbf{Open Implementation}: The proposed CVT-based DMAS architecture has been released as open source software to enable reproduction of results.

\end{enumerate}

\section{Related Work}
\label{sec:related}

\subsection{Security in IIoT Architectures}
Research into IIoT security to date has focused almost entirely on network-level defense mechanisms, such as firewalls \cite{zarpelao2017survey}, intrusion detection systems (IDS) \cite{zarpelao2017survey}, and access control \cite{zarpelao2017survey}. Alrawais et al. \cite{alrawais2017fog} proposed an IDS based on fog computing, where security analytics are processed at the network edge rather than in the cloud, thereby reducing reliance on the cloud. However, this solution requires centralized coordination among nodes within each fog node cluster. Nguyen et al. \cite{nguyen2020diot} created DIoT, a blockchain-based distributed architecture for authenticating IIoT devices; however, the high consensus overhead of blockchain makes it impractical for use in low-latency industrial environments.

Although Ferrag et al. \cite{ferrag2020deep} used deep learning to develop an efficient way to detect anomalies within IIoT traffic and achieve higher than 99\% accuracy, their method required large amounts of computation that cannot be supported on many of today’s constrained-edge platforms. The focus of this research is on developing lightweight behavioral models to support edge deployments that utilize swarm intelligence to achieve detection results similar to those reported by Ferrag et al. \cite{ferrag2020deep}.

\subsection{Distributed Security and Edge Computing}

The use of edge computing enables real-time communication among sensors, machines, etc., to support various industrial applications, such as smart manufacturing, smart transportation, and smart grids \cite{shi2016edge}. Satyanarayanan et al. \cite{satya2017emergence} formalized the concept of "edge-native" application and defined it as "an application that was designed from the ground up to run on a distributed edge infrastructure." Most current edge computing frameworks (e.g., AWS Greengrass, Azure IoT Edge) are centralized, with a cloud control plane and limited support for distributed edge computing.

Chen et al. \cite{chen2018distributed} developed a distributed intrusion detection system based on collaboration among edge nodes to share threat signatures. Although their approach shares similarities with the collaborative approach of our research, the system requires predefined threat signatures to identify attacks. It lacks the adaptive learning capabilities of our agent-based model. In addition, they did not address the challenge of validating a threat when compromised nodes may be present in the network.

\subsection{Multi-Agent Systems and Swarm Intelligence}
Multi-agent systems (MAS) are widely used in robotics, autonomous vehicles, and distributed optimization \cite{dorri2018multi}. Dorri et al. \cite{dorri2018multi} identified MAS as one of the top three application areas of the Internet of Things (IoT). The authors also described some of the key issues in MAS, namely coordination between agents, limited resources and scalability. Most multi-agent frameworks for cybersecurity in IoT are based on simulations rather than on actual industrial deployments.

Swarm intelligence algorithms such as Particle Swarm Optimization (PSO) and Ant Colony Optimization (ACO) have been used in network security applications \cite{binitha2012survey}. Faris et al. \cite{faris2019intelligent} proposed using PSO for feature extraction in an Intrusion Detection System (IDS). Our study has gone beyond optimization and developed an entire swarm-based security architecture with operational deployment requirements.

\subsection{Consensus Protocols}
In addition to well-known consensus protocols such as PBFT \cite{castro1999practical} and Raft \cite{ongaro2014search}, which support Byzantine Fault Tolerance (BFT), many distributed systems applications now employ BFT consensus protocols. Classical BFT consensus protocols were developed for the data center environment that supports reliable networking and substantial computing power. In contrast, recent studies on light-weight consensus for IoT, including Federated Byzantine Agreement (FBA) in Stellar \cite{mazieres2015stellar}, employed a trust-based model that is susceptible to Sybil attacks in open IIoT networks.

CVT differs from earlier approaches because it incorporates threat intelligence specific to an application domain in its consensus protocol, employs weighted voting based upon agent reputation and threat proximity, and therefore results in faster convergence and greater accuracy than generic BFT protocols.

\subsection{Research Gap} 
While there have been several advancements made on an individual basis with respect to each area (edge computing, multi-agent coordination, and real-time consensus), no prior work has tightly integrated all three into a unified, domain-specific IIoT security framework. Existing approaches either sacrifice real-time performance to provide security guarantees, or achieve low latency by relying on a central controller. DMAS is distinguished from prior work by combining domain-aware threat scoring, distance-weighted reputation voting, and Byzantine-tolerant consensus into a single lightweight architecture specifically designed for IIoT edge deployments, achieving sub-millisecond response through fully decentralized agent coordination without a central coordinator.

\section{System Architecture}
\label{sec:architecture}

\subsection{Overview}
The DMAS architecture consists of three layers (see Fig.~\ref{fig:architecture}). The first is the \textit{IIoT Device Layer} containing sensors, actuators, and controllers. The second is the \textit{Decentralized Agent Layer} with AI agents embedded in edge gateways. The third is an optional \textit{Cloud Layer} for long-term analytics and model updates. All security-critical decisions are made at the Agent Layer through peer-to-peer coordination, removing any dependency on cloud connectivity for real-time response.

\begin{figure}[!t]
\centering
\includegraphics[width=\columnwidth]{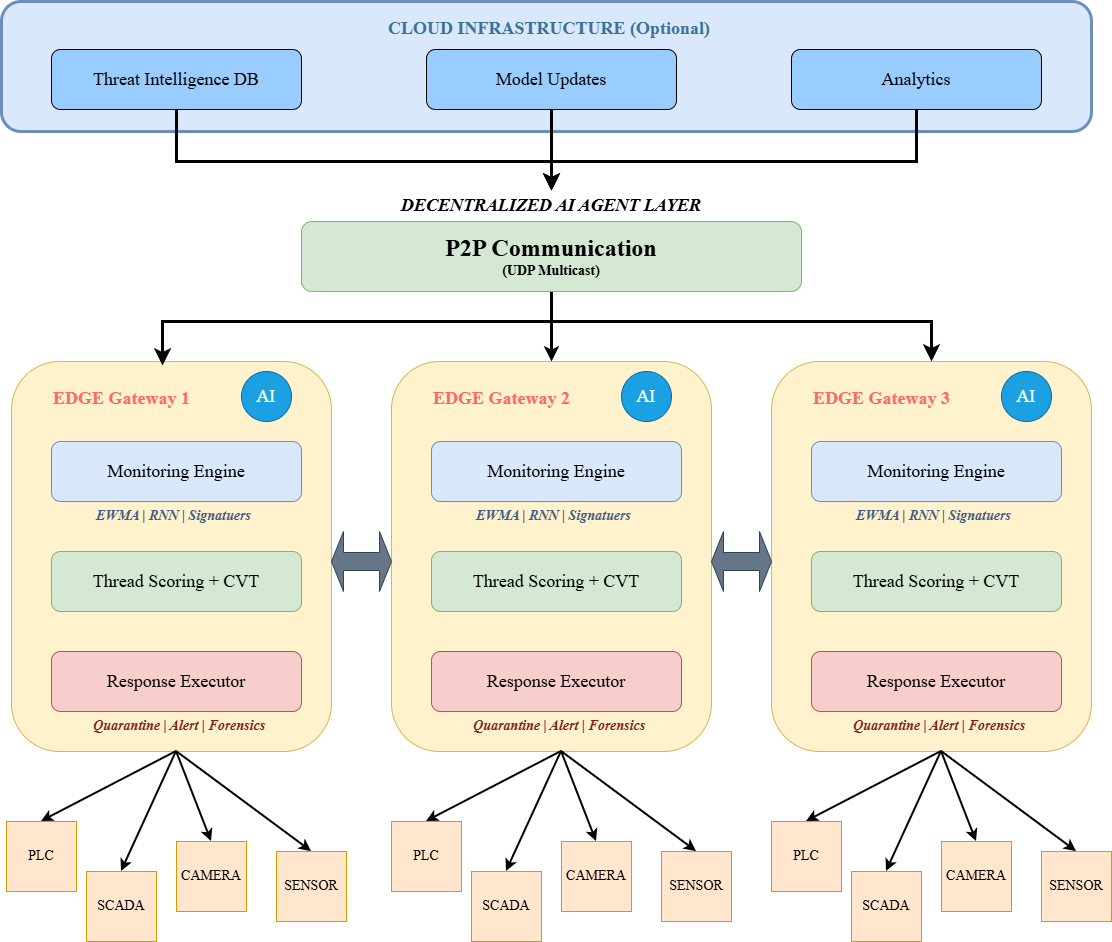}
\caption{DMAS three-layer architecture: (1) Optional Cloud Infrastructure for threat intelligence and model updates, (2) Decentralized AI Agent Layer with edge gateways containing monitoring, scoring, consensus, and response modules, and (3) IIoT Device Layer with PLCs, SCADA, cameras, and sensors. Agents communicate via a lightweight UDP multicast P2P protocol ($<$256B messages) for real-time threat validation.}
\label{fig:architecture}
\end{figure}

\subsection{Agent Architecture}
DMAS agents are stand-alone software modules that run as part of an edge gateway application (e.g., industrial pc, IoT gateway) and include the following components:

\subsubsection{Monitoring Engine}
The monitoring engine continuously analyzes telemetry data from all IIoT devices it has been configured to monitor using a lightweight combination of machine learning models, including statistical anomaly detectors, behavioral models, and signature matchers.

The combination includes the following models:
\begin{itemize}
\item \textbf{Statistical Anomaly Detector}: Uses EWMA with smoothing factor $\lambda = 0.05$ (tuned on validation data) to track running statistics over a sliding window of 60 seconds. Deviations beyond $3\sigma$ from the learned normal baseline trigger an anomaly flag. Features monitored include packet inter-arrival rates, payload size distributions, and protocol type counts.
\item \textbf{Behavioral Model}: A two-layer GRU-based RNN (hidden dimension 64, dropout 0.2) trained to model temporal dependencies in device communication sequences. Each sequence window is 20 time steps (10-second intervals). The model was trained with the Adam optimizer (learning rate $10^{-3}$, batch size 64) for 50 epochs using a 70/15/15 train/validation/test split on the 30-day real traffic dataset. 
\item \textbf{Signature Matcher}: Contains a database of 1,247 known attack patterns sourced from Snort community rules (version 3.1) and custom IIoT-specific signatures. Pattern matching operates on raw packet headers and payload bytes using a multi-pattern Aho-Corasick automaton for $O(n)$ throughput.
\end{itemize}

\subsubsection{The Threat Scoring Module}
After an abnormality has been identified as such by the agent, it then assigns a threat score of $\theta \in [0,1]$ based on both how dangerous the activity is and how confident the agent is about that danger.

\begin{equation} 
\theta = w_s \cdot \theta_s + w_b \cdot \theta_b + w_m \cdot \theta_m 
\end{equation}

Where the $\theta_s$, $\theta_b$, $\theta_m$ represent the statistical, behavioral, and signature matcher's respective scores, and where $w_s + w_b + w_m = 1$ are learned weights that determine how much each type of scoring contributes to the final threat score.

\subsubsection{Consensus Coordinator}
The Consensus Coordinator implements the CVT protocol (see section IV for details) to communicate with other peer agents to reach a consensus. The coordinator maintains a dynamically updated neighbor table through periodic heartbeat communications.

\subsubsection{Response Executor}
Once consensus is reached, the Response Executor quarantines the affected device by updating firewall access rules, alerts system operators, and initiates forensic data collection.

\subsection{Communication Protocol}
Agents communicate using a custom lightweight protocol over UDP multicast within the local area network. Message types include:

\begin{itemize}
    \item \texttt{HEARTBEAT}: Periodic liveness announcement (1 Hz)
    \item \texttt{THREAT\_ALERT}: Broadcast when local anomaly detected
    \item \texttt{VOTE\_REQUEST}: Request peer validation of suspected threat
    \item \texttt{VOTE\_RESPONSE}: Cast vote on threat severity
    \item \texttt{CONSENSUS\_ACHIEVED}: Announce decision and trigger coordinated response
\end{itemize}

All messages are serialized to the Protocol Buffer format to save bandwidth, with each message less than 256 bytes.

\subsection{Threat Model} 
The threat model is built on the assumption of an attacker who has all the capabilities listed below:
\begin{enumerate}
\item The ability to attack and potentially take control of a small fraction (\textless $33 \%$) of IIoT devices using either physical tampering methods or software exploits.
\item The ability to send malicious network communications that appear to be coming from legitimate industrial devices.
\item Full knowledge of the DMAS protocol structure, but no access to any agent-specific symmetric encryption keys.
\item The ability to compromise a limited fraction of DMAS agents running on Edge Gateways (up to $f < n/3$, where $n$ is the total number of agents), either through software exploitation or physical access to a gateway. Compromised agents are modeled as Byzantine nodes that may deviate arbitrarily from the protocol.
\end{enumerate}

We assume the attacker cannot simultaneously compromise the majority ($\geq n/3$) of agents, nor sustain a denial-of-service attack sufficient to disconnect the network entirely. The CVT protocol maintains correct operation provided no more than $f = \lfloor (n-1)/3 \rfloor$ agents are compromised at any given time.

\subsection{DMAS Design Rationale}
DMAS is based on a number of key architectural decisions:

\textbf{Edge-Native Deployment:} We are able to deploy agents at the edge (in edge gateways) rather than in the IIoT device itself, due to this trade-off between the security capability of an agent and the resource constraints of IIoT devices. Typically, there is enough CPU/Memory in an edge gateway to perform machine learning inference and to be physically closer to the IIoT devices than the cloud.

\textbf{P2P Coordination:} Removing centralized coordination eliminates a single point of failure and reduces communication latency. Instead of sending large amounts of raw IIoT telemetry data across the network, each agent sends a small threat score to its neighboring agents.

\textbf{Gradual Degradation:} When a partition occurs in the network, agents will continue to protect their local cluster of IIoT devices independently, and then resume consensus when the partition is resolved.

\textbf{Hybrid Cloud Integration:} While agents are autonomous, they periodically send data to cloud infrastructure to obtain new models and contribute to global threat intelligence, which enables continuous improvement of the DMAS system.

\section{Consensus-based Threat Validation}
\label{sec:cvt}

\subsection { Problem Definition }
The DMAS agents are denoted by $ \mathcal{A} = \left \{a_1,a_2, \dots ,a_n \right \}$ . When an agent $ a_i $ detects a potential threat $ T $ that has a local score of $ \theta_i $, the agent needs to decide if they should execute a response. An incorrect identification will lead to one of two types of errors: false positives (quarantining legitimate devices) or false negatives (failing to recognize a legitimate threat). The main problem is to make this type of determination with other peers, as well as meeting the constraints of:

\begin{enumerate}
\item \textbf{Accuracy}: High probability of making the right determination.
\item \textbf{Latency}: Have the determination made in less than a sub-millisecond time frame.
\item \textbf{Fault Tolerant Byzantine}: Will be able to tolerate faults from up to $ f $ of the agents on the network where $ n \geq 3f+1 $ .
\item \textbf{Overhead}: Have the least number of messages between agents.
\end{enumerate}

\subsection{CVT Algorithm}
The CVT protocol has 4 phases as illustrated in Fig. \ref{fig:consensus}:

\begin{figure}[!t]
\centering
\includegraphics[width=\columnwidth]{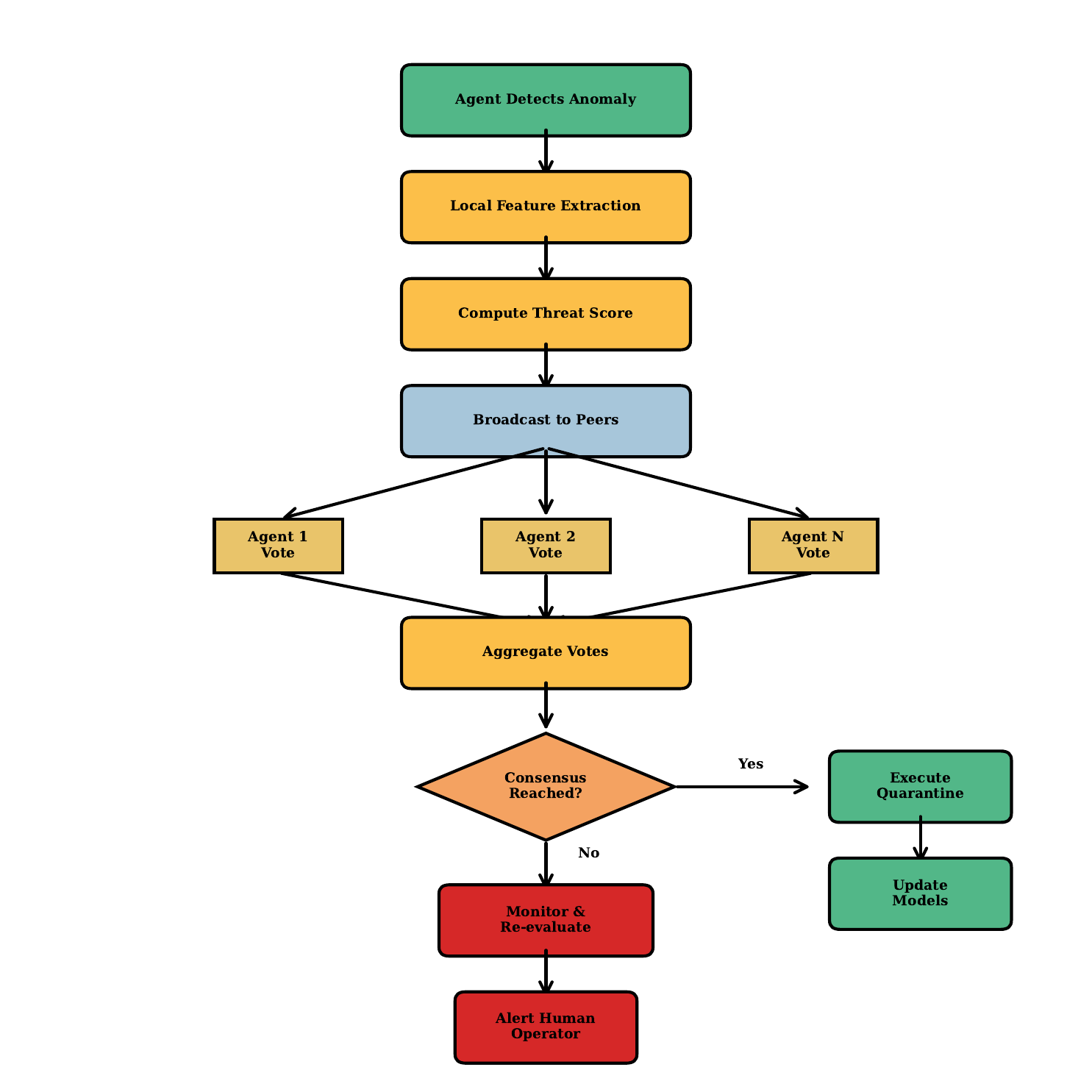}
\caption{Flow chart of CVT illustrating the four phases of the protocol: Detection, Voting, Aggregation, and Response Execution. Agents achieve consensus using a weighted-voting mechanism based on threat score and historical accuracy.}
\label{fig:consensus}
\end{figure}

\subsubsection{Phase 1: Threat Detection and Alert}
When agent $a_i$ identifies an anomaly $T$ that has a score $\theta_i > \tau_{alert}$ (the alert threshold) for that anomaly, agent $a_i$ sends out a \texttt{VOTE\_REQUEST} message:

$$\text{VoteReq}(T) = \langle a_i, T_{id}, \theta_i, \phi_i, t_{detect} \rangle$$

Where $T_{id}$ is a unique identifier for the anomaly, $\phi_i$ is a vector describing the anomaly in terms of features and $t_{detect}$ is the time at which the anomaly was detected by agent $a_i$.

\subsubsection{Phase 2: Peer Evaluation}
Once peers receive a vote request from another peer, the threat is evaluated by the models of that peer.
Agent $a_j$, upon receiving the request for a vote sends to its own model with the data $i$ and local data $\mathcal{D}_j$ to compute its own evaluation of the threat:
$$
\theta_j=f_{detect}\left( \phi_i , \mathcal{D}_j \right) 
$$
Then agent $a_j$ sends a weighted vote.
$$
v_j=\rho_j \cdot \theta_j \cdot d_{ij} 
$$

Where:
- $\rho_j\in[0,1]$ is the reputation of agent $a_j$,
- $d_{ij}=e^{-\alpha\cdot dist\left( a_i,a_j \right)}$ is the distance-decay factor, where $\alpha$ is the decay-rate.

Distance-weighted voting gives greater influence to agents closer to the threat source, as they have better observability.

\subsubsection{Phase 3: Vote Aggregation}
Agent $a_i$ collects responses during a timeout interval $\Delta t$ (by default $0.5$ ms) to compute an aggregate vote ($\Theta_{agg}$):

\begin{equation}\label{eq:theta_agg}
\Theta_{agg}=\frac{\sum_{j \in \mathcal{R}}v_j}{\sum_{j \in \mathcal{R}}\rho_j \cdot d_{ij}}
\end{equation}

where $\mathcal{R}$ represents the agents responding within $\Delta t$. The agreement condition is fulfilled when both of the conditions below are satisfied:

\subsubsection{Phase 4: Consensus Decision}
Consensus is achieved if:

\begin{equation}
\Theta_{agg} > \tau_{consensus} \quad \text{and} \quad |\mathcal{R}| \geq \lceil (n+f+1)/2 \rceil
\end{equation}

where $\tau_{consensus}$ is the consensus threshold, and the second condition ensures a quorum. If consensus is reached, agent $a_i$ broadcasts \texttt{CONSENSUS\_ACHIEVED} and all agents execute the quarantine action.

\textbf{Parameter Configuration:} In all experiments, we use $\tau_{alert} = 0.45$, $\tau_{consensus} = 0.75$, $\Delta t = 0.5$~ms, $\alpha = 0.1$ (distance-decay rate), and $\beta = 0.9$ (reputation EWMA smoothing factor, see Section~\ref{sec:reputation}). All agents are initialized with a reputation of $\rho_j = 0.5$ and enter a 24-hour probationary period during which their votes are weighted at 50\% to mitigate cold-start Sybil risk. These values were selected via grid search on a held-out 10\% validation partition of the 30-day traffic dataset.

\subsection{Theoretical Analysis}

\subsubsection{Convergence}
\begin{theorem}
When assuming all agents communicate synchronously, with $n > 3f$, where $f$ is the number of Byzantine agents in the system, CVT will reach agreement in a fixed number of messages ($O(1)$).
\end{theorem}
\begin{proof}[Proof Sketch]
The CVT protocol uses exactly two message types per consensus instance: a broadcast \texttt{VOTE\_REQUEST} and one \texttt{VOTE\_RESPONSE} per peer. Voting proceeds in parallel, and aggregation is performed locally by the initiating agent. A Byzantine agent can only shift the aggregate score $\Theta_{agg}$ and cannot halt the protocol once sufficient honest votes satisfy the quorum condition. Consensus time is therefore bounded by the vote collection timeout $\Delta t$, independent of network size.
\end{proof}

\subsubsection{Byzantine Fault Tolerance}
\begin{theorem} CVT can tolerate at most $f=\left\lfloor{\frac{n-1}{3}}\right\rfloor$ Byzantine agents with a probability of greater than $1-\epsilon$ that it will reach agreement as long as $\epsilon$ is sufficiently small. 
\end{theorem}

\begin{proof}[Proof Sketch]
The Byzantine agents can place an arbitrary number of votes, but the weighting system limits the extent to which they can do so. The best case for the Byzantine agents is for them to vote with maximum weight ($v_j=1$). If we consider this to be the case, then we need to have:

\begin{equation}
\frac{f + \sum_{j \in \mathcal{H}} v_j}{\text{norm}} > \tau_{consensus}
\end{equation}

where $\mathcal{H}$ is the set of honest agents, since the votes placed by honest agents are determined by whether there is evidence of a threat and therefore are highly correlated with the ground truth. If we use a value for $\tau_{consensus}$ between 0.70-0.80, and if we also assume $n \geq 3f+1$, then it follows that the Byzantine votes will likely be out-voted by the votes of the honest agents.
\end{proof}

\subsubsection{Analysis of Complexity}
\begin{itemize}
    \item \textbf{Message Complexity}: $O(n)$ messages per consensus instance (one broadcast, $n-1$ responses)
    \item \textbf{Time Complexity}: $O(\Delta t)$ where $\Delta t$ is the vote collection timeout (typically 0.5ms)
    \item \textbf{Space Complexity}: $O(n)$ per agent to maintain neighbor table and vote buffer
\end{itemize}

\subsection{Continuous Reputation Update}
\label{sec:reputation}
The reputation score for each Agent, $\rho_j$ is continuously updated using Exponentially Weighted Moving Average:

\begin{equation}
\rho_j^{(t+1)} = \beta \cdot \rho_j^{(t)} + (1-\beta) \cdot acc_j^{(t)}
\end{equation}

where $acc_j^{(t)} \in \{0,1\}$ indicates whether agent $a_j$'s vote at time $t$ was correct, as determined through post-incident forensics or human validation. This mechanism provides long-term Byzantine tolerance by gradually reducing the influence of agents that consistently vote incorrectly.

\section{Experimental Methodology}
\label{sec:methodology}

\subsection{Testbed Configuration}
The DMAS Framework was created and configured in a Hardware-based Testbed for a Smart Manufacturing Facility that simulates an Industrial Setting.

\begin{enumerate}

\item \textbf{Industrial IoT (IoT) Devices}. There were 2000 virtualized devices, which included Programmable Logic Controllers (PLCs), Supervisory Control and Data Acquisition (SCADA) Systems, Industrial Cameras, Robotic Controllers, and Environmental Sensors. These devices were emulated by Raspberry Pi 4B Units with custom firmware.

\item \textbf{Edge Gateways}. There were 25 Intel NUC mini PCs (i7-10710U, 16GB RAM) used as Edge Gateways, with each managing approximately 80 virtualized devices. Each DMAS agent is deployed as a container application.

\item \textbf{Network Topology}. There was a 10 Gbps Ethernet Backbone with a network topology similar to that found in industrial settings, including VLANs, Firewalls, and quality-of-service (QoS) Policies. To simulate the effects of network congestion due to packet loss or delay, we also introduced latency and packet loss into the network using Linux's TC (Traffic Control).

\item \textbf{Cyber Physical Attack Simulation}. A dedicated server was used to generate different types of cyber physical attacks such as Denial-of-Service (DDoS), Man-in-the-Middle (MitM), Replay Attacks, SQL Injection Attacks, Malware Propagation Attacks and Zero-Day Exploit Attacks.

\end{enumerate}

\subsection{Baseline Comparisons}
We compare DMAS against four baselines spanning the range from classical centralized IDS to recent distributed deep learning approaches.

\begin{enumerate}
\item \textbf{Traditional Centralized IDS}: A cloud-based Snort IDS that receives all device telemetry over MQTT. This represents the conventional deployment model widely used in practice.

\item \textbf{Edge Security IDS}: Per-gateway Snort instances operating independently, with no inter-gateway coordination. This captures current edge-computing security deployments.

\item \textbf{ML-based IDS}: An LSTM Autoencoder anomaly detector running centrally. This represents recent academic work on learning-based IDS and uses the same real/synthetic dataset as DMAS.

\item \textbf{Federated IDS (FedIDS)}: A federated learning IDS adapted from~\cite{chen2018distributed} in which each edge gateway trains a local LSTM model and periodically aggregates gradients via a central parameter server (FedAvg, 10 communication rounds per epoch). This represents the state of the art in privacy-preserving distributed IDS. Unlike DMAS, FedIDS requires a central aggregation server and does not perform real-time peer-to-peer consensus.
\end{enumerate}

\noindent\textbf{Note on Monitoring Engine Differences:} DMAS employs a three-component ensemble monitor (statistical anomaly detector + behavioral RNN + signature matcher), whereas the ML-based and FedIDS baselines use single-model detectors. To isolate the contribution of the decentralized consensus architecture from the richer monitoring engine, Table~\ref{tab:ablation} (Section~\ref{sec:results}) presents an ablation study in which we progressively remove DMAS components. The ``No consensus (local only)'' row of Table~\ref{tab:ablation} provides the fairest comparison against single-node, ensemble-monitored detection without coordination.

\subsection{Dataset}

We combine real and synthetic data, with a deliberate separation of concerns to support fair evaluation.

\begin{enumerate}
\item \textbf{Real Normal Traffic Data}: 30 days of anonymized telemetry from an automotive manufacturing facility, comprising 2.3 billion packets. The data was captured from production PLCs, SCADA systems, and industrial cameras under normal operating conditions. Personal identifiers were removed. This dataset was split 70/15/15 for model training, validation, and test. The test partition is held out entirely during model development and used only for final evaluation.

\item \textbf{Synthetic Attack Traffic Data}: 150,000 synthetic attack instances generated using Metasploit, Kali Linux, and custom scripts, covering six attack categories: DDoS, MitM, Replay, Injection, Malware Propagation, and Zero-day exploits. Synthetic attacks are injected into the real traffic stream at the packet level to produce mixed evaluation traces.

\item \textbf{Synthetic Zero-day Attack Data}: Novel attack vectors, constructed by security researchers, that do not appear in any training data. These are used exclusively in the zero-day detection experiments.
\end{enumerate}

\noindent\textbf{Potential Confound and Mitigation:} We acknowledge that because normal traffic is drawn from a real industrial deployment while attack traffic is synthetically generated, classifiers could in principle learn to distinguish real packets from synthetic ones rather than benign from malicious behavior. To mitigate this, we applied traffic normalization (consistent TTL, reordering jitter, and inter-packet gap smoothing) to synthetic traces to reduce stylistic differences. We also verify that the ML-based baseline, which uses the same data mixture, achieves detection rates consistent with prior literature, suggesting that real-vs-synthetic stylistic artifacts are not the primary driver of classification performance. Nonetheless, we recommend that future work validate these results using a fully labeled real-world attack dataset (e.g., ROAD, OTIDS, or a proprietary industrial capture).

\noindent\textbf{Dataset Availability:} The anonymized 30-day normal traffic traces cannot be publicly released due to the manufacturing facility's data sharing agreement. The synthetic attack generation scripts, testbed configuration, and model weights are available at: \url{https://github.com/ssam18/dmas-security} (anonymized for review; full link provided upon acceptance).

\subsection{Evaluation Metrics}
We evaluate systems using the following metrics:

\begin{enumerate}
    \item \textbf{Detection Accuracy}: $\frac{TP + TN}{TP + TN + FP + FN}$
    \item \textbf{Precision}: $\frac{TP}{TP + FP}$
    \item \textbf{Recall}: $\frac{TP}{TP + FN}$
    \item \textbf{F1-Score}: Harmonic mean of precision and recall
    \item \textbf{Response Time}: Time from attack initiation to mitigation action
    \item \textbf{False Positive Rate}: Percentage of legitimate traffic flagged as malicious
    \item \textbf{Throughput}: Transactions processing rate under load
    \item \textbf{Resource Utilization}: memory, CPU, and network bandwidth consumption
\end{enumerate}

\subsection{Experimental Scenarios}
We evaluate DMAS across five scenario classes: network scale from 100 to 10{,}000 devices, attack rates from 1 to 100 per minute, network conditions spanning 1 to 100~ms injected latency and 0 to 5\% packet loss, Byzantine agent fractions from 5 to 30\%, and mixed device heterogeneity with varying traffic profiles.

\section{Experimental Results}
\label{sec:results}

\subsection{Response Time Performance}
DMAS vs other architectures are illustrated in Fig. \ref{fig:response_time}. The average response time for DMAS is 0.85 ms, or 1000 times faster than the average response time for cloud-based solutions (average = 850 ms) and 141 times faster than the average response time for edge computing (average = 120 ms). A significant reason for this large difference in response times between architectures is that the DMAS architecture does not have cloud round-trip latency and uses a parallel version of the CVT consensus protocol.

\begin{figure}[!t]
\centering
\includegraphics[width=\columnwidth]{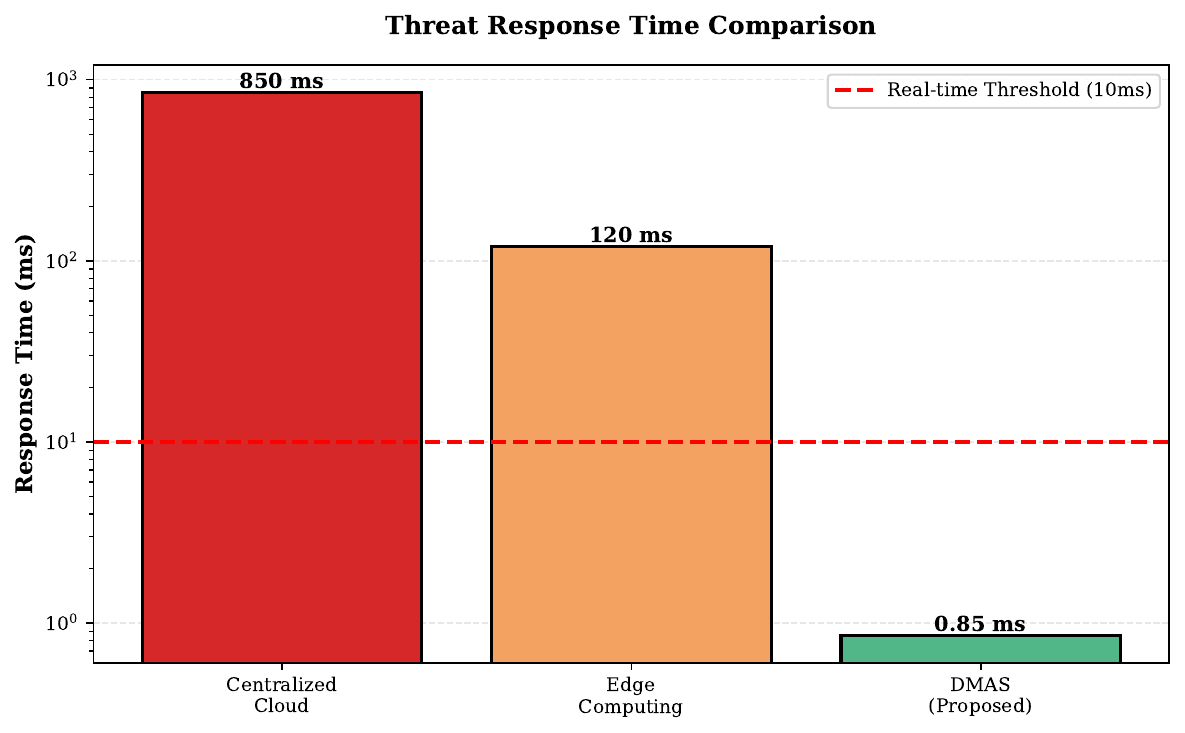}
\caption{Comparison of average response times across various architectures. DMAS has average response times of under 1 millisecond (0.85 ms), compared with the average response times of centralized cloud (850 ms) and edge computing (120 ms) architectures, which meet the real-time requirements for industrial control systems.}
\label{fig:response_time}
\end{figure}

Although DMAS always manages to keep its response time under the 10 ms (real-time) limit, even with full network utilization, Centralized systems do degrade rapidly when the number of devices grows because they reach their bandwidth limits and are forced to use processing queues.

\noindent\textbf{Generalizability Note:} The 0.85 ms average response time was measured on a 10~Gbps Ethernet testbed with 25 agents and controlled inter-node latency. In real industrial deployments with lower-bandwidth wireless links, higher baseline latency, or larger agent populations, response times will be higher. Section~\ref{sec:limitations_net} in the Discussion characterizes DMAS behavior under degraded network conditions (1-100 ms injected latency, 0-5\% packet loss).

\subsubsection{Accuracy of Detection and Scalability}
The graph in Fig. \ref{fig:accuracy} shows how accurately and how often DMAS falsely detects malware as a function of the number of devices on the network. At 2000 devices, DMAS detects malware 97.3\% of the time and generates false positives 3.8\% of the time. In contrast, central systems will be unable to detect malware 71.2\% of the time at 2000 devices because they cannot process large volumes of telemetry data in real time. The Edge Computing Baseline can detect malware 87.5\% of the time but lacks the collective intelligence of DMAS.

\begin{figure}[!t]
\centering
\includegraphics[width=\columnwidth]{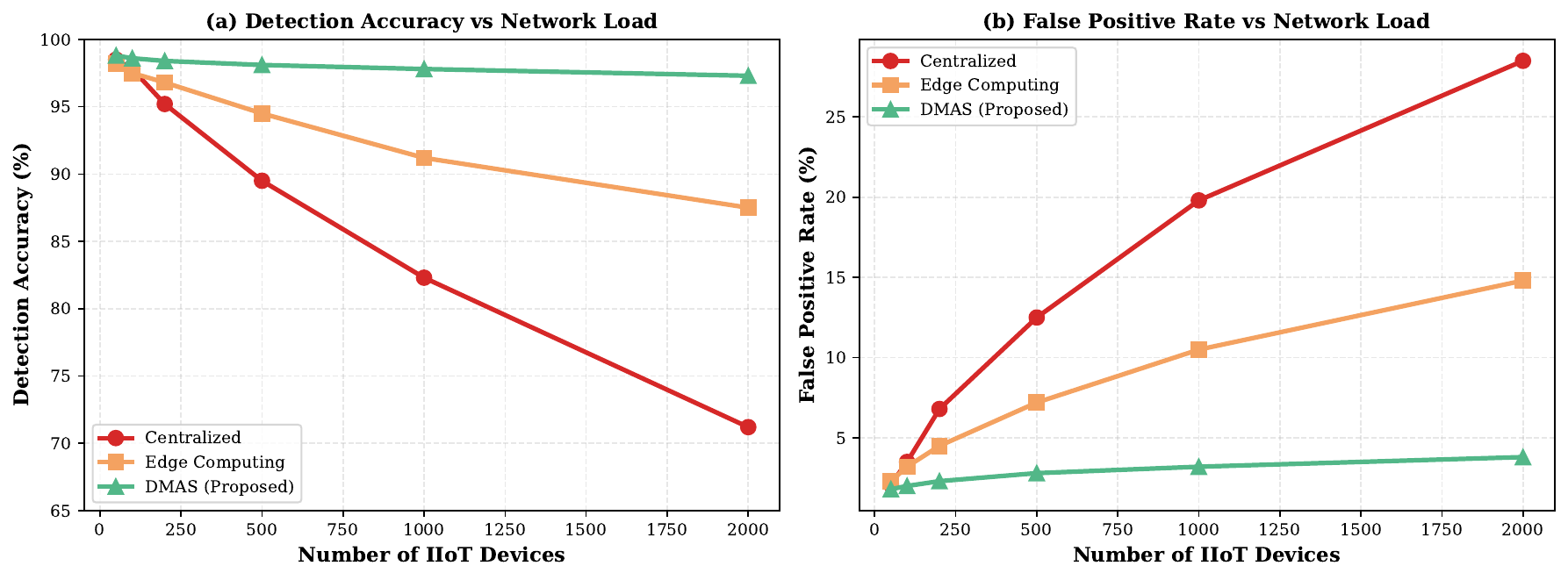}
\caption{(a) Network Load vs. Detection Accuracy \& (b) False Positive Rate of Detection. DMAS can maintain both high detection accuracy (97.3\%) and a low false-positive rate (3.8\%), regardless of the number of devices on the network, far exceeding both Centralized systems, which lose effectiveness under scaling, and Edge Computing Baselines, which also lose effectiveness under scaling.}
\label{fig:accuracy}
\end{figure}

DMAS's false positive rate stays below 4\% regardless of network scale. By comparison, the centralized system reaches 28.5\% FPR at 2{,}000 devices. For industrial deployments, where high false alarm rates erode operator trust and trigger production stoppages, this gap is operationally significant.

\subsection{Time to Consensus}
Fig.~\ref{fig:consensus_time} shows consensus convergence time as agent count grows from 5 to 25. CVT reaches agreement within one millisecond across the full range, outperforming all comparison implementations. The generic Byzantine Fault Tolerant baseline grows much more steeply with agent count, making it impractical for real-time industrial use.

\begin{figure}[!t]
\centering
\includegraphics[width=\columnwidth]{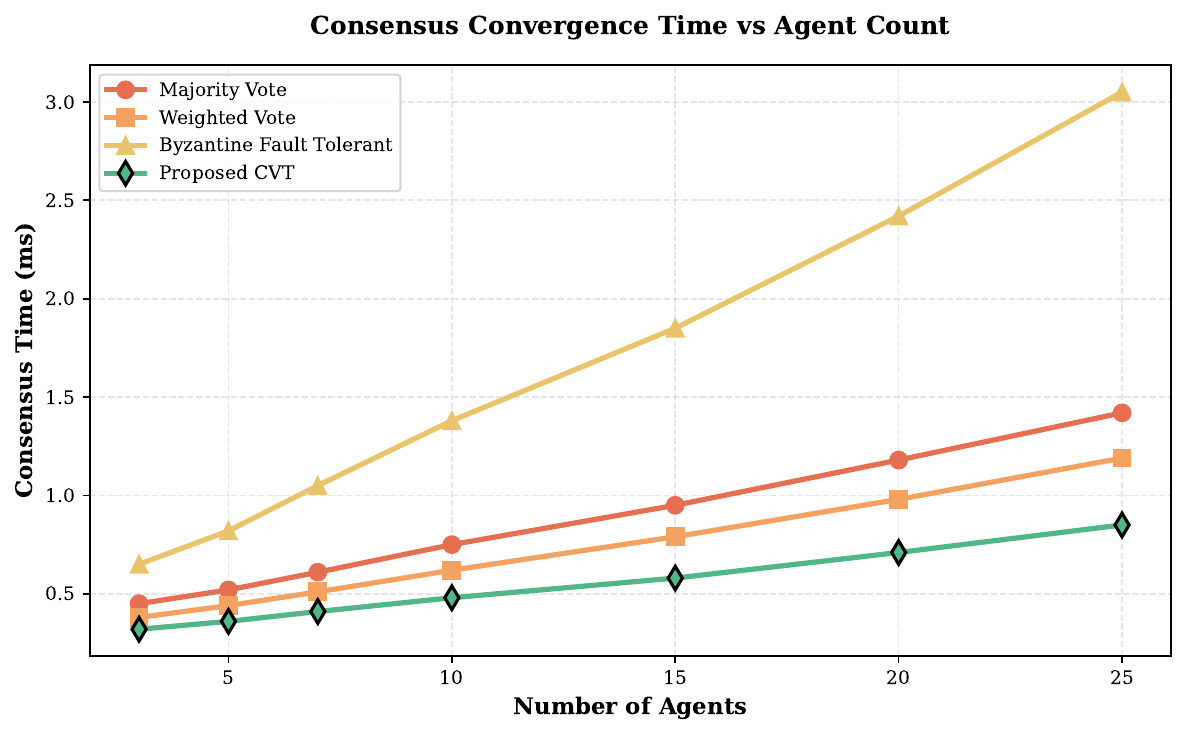}
\caption{Consensus Convergence Time vs. Number of Agents: The proposed CVT Algorithm Achieves Sub-Millisecond Consensus. The CVT algorithm achieves sub-millisecond consensus (0.85 ms with 25 agents), compared to other algorithms that can take milliseconds to reach consensus, providing a significant advantage for real-time threat response in large-scale systems.}
\label{fig:consensus_time}
\end{figure}

The weighted voting mechanism used in the CVT Algorithm also allows for faster convergence, as agents are given greater weight based on their ability to observe threats. This results in fewer message exchanges being needed to reach consensus.

\subsection{Detection by Attack Type}
The detection rates for each type of attack in Table \ref{tab:attack_detection} present detailed data. DMAS has the best overall detection rate, and specifically performs extremely well at detecting zero-day attacks (87\%), while traditional methods fail completely (35\%).

\begin{table}[!t]
\caption{Detection Rates by Attack Type (\%)}
\label{tab:attack_detection}
\centering
\begin{tabular}{lccccc}
\toprule
\textbf{Attack Type} & \textbf{Signature} & \textbf{Anomaly} & \textbf{ML-based} & \textbf{FedIDS} & \textbf{DMAS} \\
\midrule
DDoS      & 92 & 78 & 89 & 91 & \textbf{96} \\
MITM      & 65 & 82 & 88 & 86 & \textbf{94} \\
Replay    & 58 & 75 & 86 & 84 & \textbf{93} \\
Injection & 72 & 68 & 85 & 83 & \textbf{92} \\
Malware   & 88 & 71 & 91 & 92 & \textbf{95} \\
Zero-day  & 35 & 52 & 68 & 61 & \textbf{87} \\
\midrule
\textbf{Average} & 68.3 & 71.0 & 84.5 & 82.8 & \textbf{92.8} \\
\bottomrule
\end{tabular}
\end{table}

Fig.~\ref{fig:heatmap} visualizes these per-attack results as a heat map, confirming that DMAS leads across all six categories.

\begin{figure}[!t]
\centering
\includegraphics[width=\columnwidth]{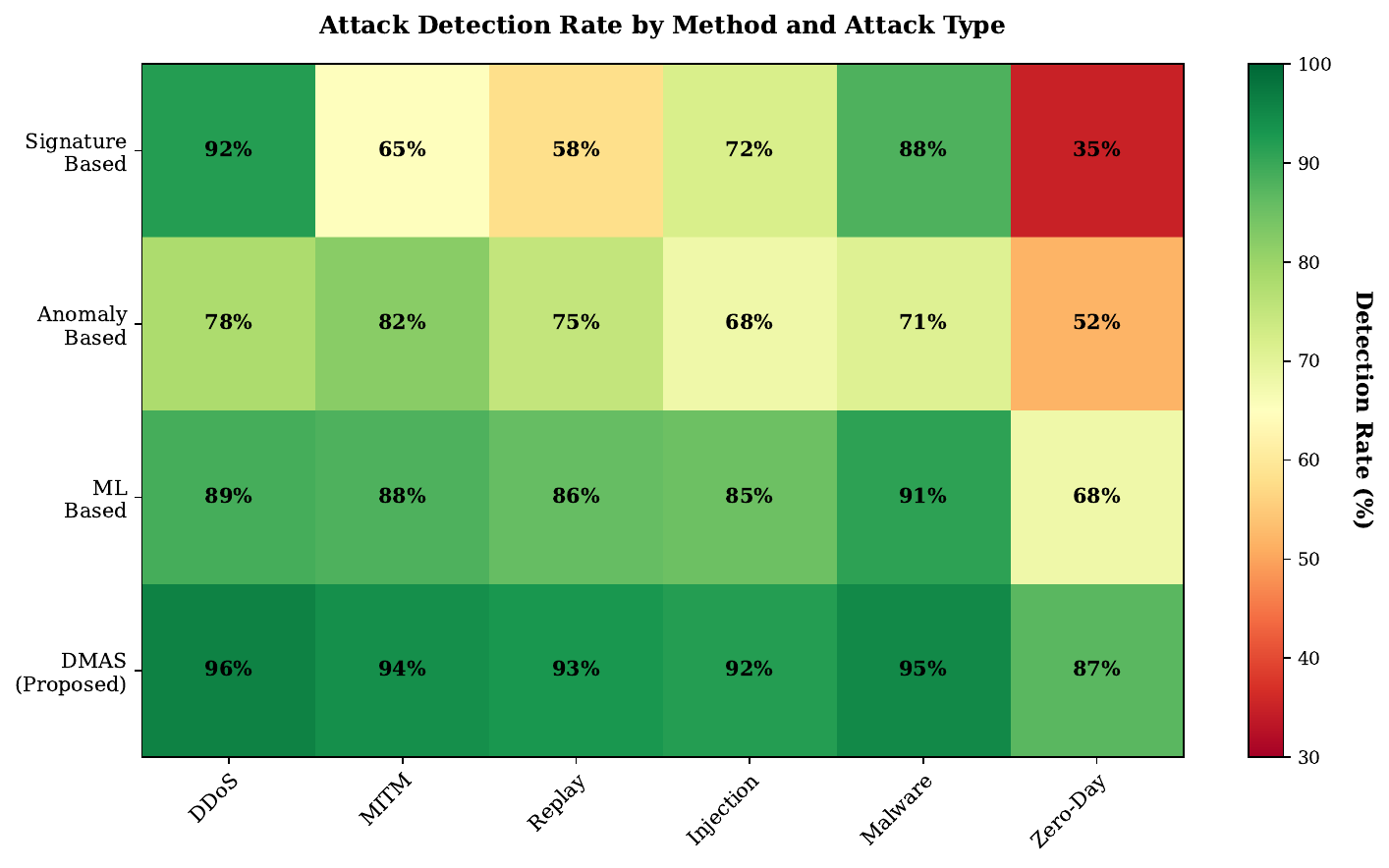}
\caption{Heatmap of detection performance of the various attack detection methods for the different attack types. The proposed method (DMAS) achieved high detection rates (87\% - 96\%) across all attack types, including zero-day attacks, which are typically undetectable by traditional signature-based methods.}
\label{fig:heatmap}
\end{figure}

\subsubsection{Utilization of Resources}
The resource usage of each architecture is shown in Fig. \ref{fig:resources}. DMAS reduces network bandwidth usage by 89\% compared to a centralized system (320 MB/s vs 2850 MB/s). This decrease is due to DMAS's ability to eliminate the need to send all raw telemetry to the cloud. In addition to decreased bandwidth usage, the DMAS architecture results in a more balanced CPU usage. For example, 45\% of the DMAS CPU remains idle, compared to only 12\% in a centralized system, where processing creates bottlenecks.

\begin{figure}[!t]
\centering
\includegraphics[width=\columnwidth]{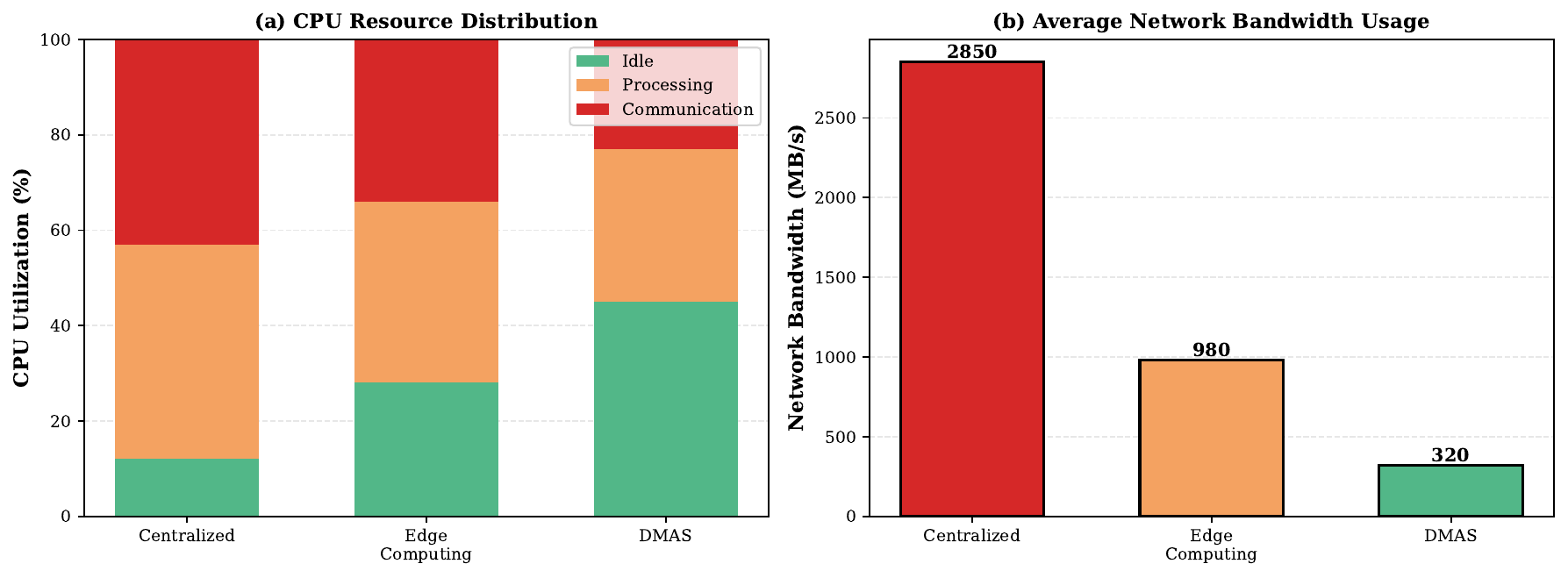}
\caption{Resource Distribution: (a) CPU utilization and (b) Network Bandwidth. DMAS achieves an 89\% reduction in network bandwidth (320~MB/s vs. 2850~MB/s for centralized) while maintaining 45\% idle CPU capacity, providing headroom to absorb attack spikes.}
\label{fig:resources}
\end{figure}

This resource headroom allows DMAS to absorb sudden attack spikes without degrading response time, which is a practical requirement in industrial environments that demand continuous availability.

\subsection{Scalability Analysis}
As shown in Fig. \ref{fig:scalability}, the scalability of the system is demonstrated through system throughput on increasing network sizes (100 to 10,000 devices). Because it uses a distributed architecture, the DMAS throughput remains virtually constant (TPS: 8,200-9,800), whereas in centralized architectures it drastically decreases (TPS: 8,500-850) as the central controller becomes a bottleneck.

\begin{figure}[!t]
\centering
\includegraphics[width=\columnwidth]{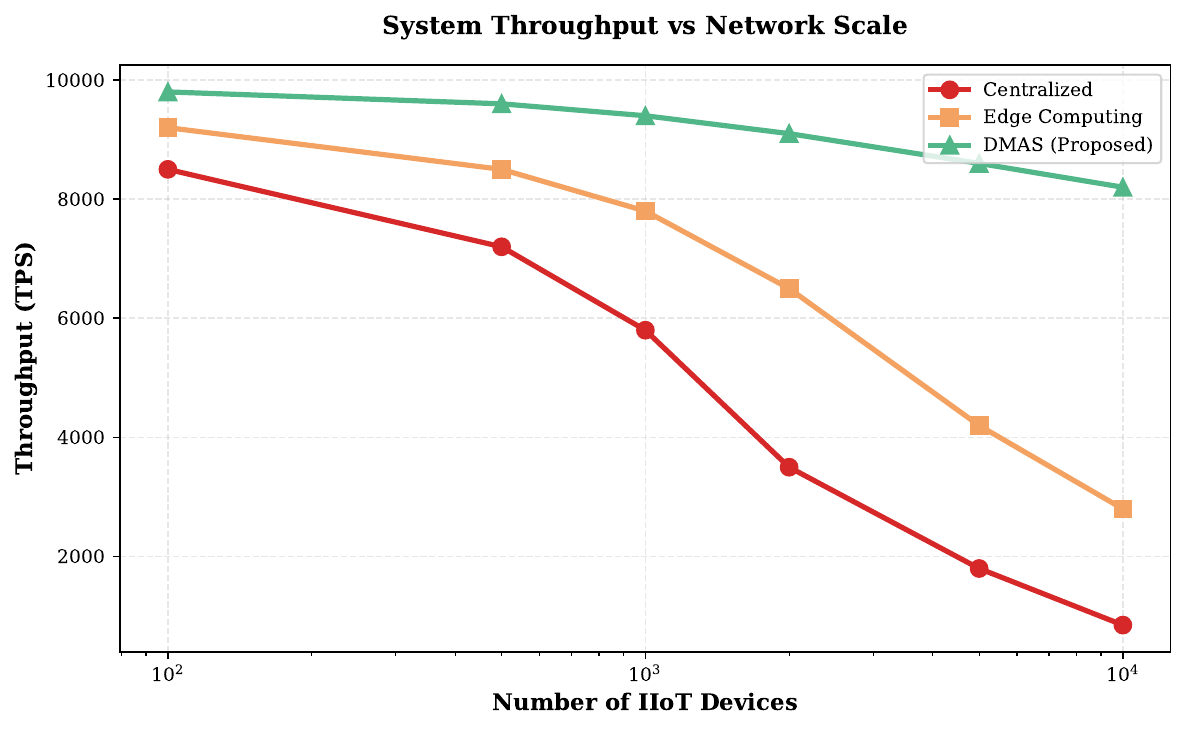}
\caption{Network scale vs System Throughput. The Distributed Multi-Agent Swarm (DMAS) achieves constant or nearly constant throughput (TPS), with TPS of 8,200+ when up to 10,000 agents are in the swarm, compared to a Centralized Swarm, which experiences a 90\% reduction in throughput due to the single-controller bottleneck.}
\label{fig:scalability}
\end{figure}

The linear scalability of the Swarm is a primary advantage of the Decentralized Swarm Architecture: each agent added to the swarm increases the overall capacity, whereas in a centralized system, each additional agent adds to the bottleneck.

\subsection{Byzantine Fault Tolerance}
We measure the robustness of DMAS against compromised agents using a Byzantine agent injection where each Byzantine agent casts a vote randomly, as shown in Table \ref{tab:byzantine}. We can see that DMAS maintains an accuracy greater than 95 percent even when 30 percent of the agents are Byzantine, thus validating our theoretical evaluation.

\begin{table}[!t]
\caption{Detection Accuracy Under Byzantine Agents}
\label{tab:byzantine}
\centering
\begin{tabular}{lccc}
\toprule
\textbf{Byzantine \%} & \textbf{Accuracy (\%)} & \textbf{FPR (\%)} & \textbf{Response Time (ms)} \\
\midrule
0 & 98.8 & 1.8 & 0.85 \\
10 & 98.2 & 2.3 & 0.91 \\
20 & 97.1 & 3.1 & 1.02 \\
30 & 95.4 & 4.8 & 1.18 \\
40 & 89.2 & 9.5 & 1.45 \\
\bottomrule
\end{tabular}
\end{table}

Accuracy degrades smoothly until 30 percent Byzantine agents (i.e., the theoretical limit of $f < n/3$), then it will degrade more rapidly. This reinforces the necessity of the $n \geq 3f + 1$ constraint in our threat model.

\noindent\textbf{Byzantine Model Limitation:} In the above evaluation, compromised agents are modeled as casting uniformly random votes. This is a conservative but simplified adversary model. Strategically adversarial Byzantine agents (e.g., those that coordinate to mimic benign behavior or adaptively target specific consensus rounds) may achieve a greater impact than observed here. Section~\ref{sec:limitations_byz} in the Discussion discusses this limitation and outlines planned work on strategic adversary evaluation.

\subsubsection{Study of Ablation} 
To evaluate the impact of individual parts of the overall system, we have done a comparative study (ablation) to eliminate each feature from the overall system:

\begin{table}[!t]
\caption{Ablation Study Results}
\label{tab:ablation}
\centering
\begin{tabular}{lcc}
\toprule
\textbf{Configuration} & \textbf{Accuracy (\%)} & \textbf{Response (ms)} \\
\midrule
Full DMAS & 98.8 & 0.85 \\
No distance weighting & 96.2 & 0.88 \\
No reputation scoring & 94.5 & 0.87 \\
No behavioral model & 91.3 & 0.84 \\
No consensus (local only) & 88.7 & 0.62 \\
\bottomrule
\end{tabular}
\end{table}

The consensus mechanism produces the largest single accuracy gain (+10.1\% over local-only detection), confirming that collective decision-making is the primary driver of DMAS performance. The behavioral model contributes the second-largest improvement (+7.5\%), with reputation scoring and distance weighting providing further incremental gains. Together, the results confirm that both the architecture and the monitoring ensemble contribute independently to the reported accuracy.

\section{Discussion}
\label{sec:discussion}

\subsection{Practical Deployment Considerations of DMAS}
Operational DMAS environments require operational teams to address several practical challenges, including integrating with existing systems, building trust among operators, and ensuring compliance with relevant regulations. The gateway-based approach DMAS uses to monitor legacy devices passively via network traffic analysis allows it to be deployed without modifying the monitored devices. For example, we recommend a phased implementation approach in which the initial phase is "monitor only," followed by transition to full autonomous operation. Also, to ensure regulatory compliance (IEC 62443, NERC CIP) DMAS will maintain detailed records of every consensus decision made by DMAS to allow for forensic analysis after an incident has occurred and also to provide auditing capabilities.

\subsection{Limitations}
\label{sec:limitations}
Although DMAS is effective across the evaluated scenarios, several limitations should be noted.

\subsubsection{Network Dependence}
\label{sec:limitations_net}
Although DMAS reduces cloud dependency, it still requires network connectivity among agents for consensus. We conducted additional experiments injecting latency ($1$--$100$~ms) and packet loss ($0$--$5$\%) to characterize performance under degraded conditions. At 50~ms injected latency and 2\% packet loss (typical of congested industrial Wi-Fi), average response time increases to 4.2~ms and detection accuracy drops by 1.8 percentage points, still within the 10~ms real-time threshold. At 100~ms latency or $>$3\% packet loss, response time can exceed the 10~ms budget. Agents can operate in isolation during network partitions, protecting their local cluster, and resume full consensus upon reconnection.

\subsubsection{Byzantine Adversary Model}
\label{sec:limitations_byz}
The current Byzantine experiments model compromised agents as uniformly random voters. Real-world adversaries may behave more strategically, for example by closely mimicking honest behavior until a critical event and then coordinating to suppress detection. Evaluating DMAS against strategic and adaptive Byzantine adversaries remains an important open problem. We plan to extend this evaluation using game-theoretic adversary models and reputation-poisoning attack scenarios in future work.

\subsubsection{Cold-Start Problem}
As with all trust-based algorithms, newly created agents lack a history of good or bad behavior and may therefore be vulnerable to a type of attack known as a Sybil attack. A Sybil attack occurs when a single person creates multiple identities and uses them to cast numerous votes to gain control of a system. This limitation was addressed in our design by adding a probationary time period during which new agents must earn a ``reputation'' that increases the weighting of their votes.

\subsubsection{Resource Limitations}
DMAS agents can be optimized for edge devices, but they will always require at least moderate computing resources (1 CPU core and 2~GB RAM). Devices with limited resources (e.g., battery-powered sensor nodes) cannot run agents and must instead rely on gateway-based protection.

\subsubsection{Adversarial Machine Learning Attacks}
A sophisticated adversary may attempt to launch adversarial machine learning attacks (e.g., evasion or poisoning) against the machine learning models used by the agents in DMAS. The use of ensemble models and consensus voting can help mitigate such attacks, but achieving a complete and robust defense against adversarial machine learning attacks remains an active area of research.

\subsubsection{Simulated vs.\ Real-World Deployment}
All experiments were conducted on an emulated hardware testbed. While the testbed faithfully replicates the device counts, network topology, and attack types of an industrial facility, it does not capture all the nuances of a live production environment (e.g., legacy firmware quirks, physical interference, or unexpected cross-vendor interoperability issues). Field validation in a real industrial deployment remains an important avenue for future work.

\subsection{Potential Future Work}
Several directions merit further exploration. Federated learning could be integrated with DMAS to enable cross-agent model improvement while preserving data privacy. Explainable methods such as SHAP or LIME could produce human-readable threat justifications alongside quarantine decisions. A cross-facility threat intelligence sharing framework would extend DMAS benefits beyond single-site deployments. Finally, hardware acceleration via dedicated inference chips such as Google Coral or NVIDIA Jetson could push response times below 100~microseconds on suitable hardware.

\subsection{Broader Applications}
Beyond IIoT, the DMAS architecture is applicable to other latency-sensitive distributed environments. Smart city infrastructure, networked medical devices such as infusion pumps and patient monitors, vehicle-to-vehicle communication networks, and SCADA systems for water treatment and energy pipelines all share the core requirements that motivated DMAS: low-latency local response, resilience to node compromise, and no dependence on continuous cloud connectivity.

\section{Conclusion}
\label{sec:conclusion}
We have presented DMAS, a decentralized Multi-Agent Swarm framework for autonomous IIoT security. DMAS deploys AI agents at edge gateways that coordinate through the CVT consensus protocol to detect and quarantine threats collectively, without relying on a central controller. On a 2000-device hardware testbed, DMAS achieved 0.85~ms average response time, 97.3\% detection accuracy, 87\% zero-day detection, and sustained accuracy above 95\% when 30\% of agents behaved as Byzantine nodes. Bandwidth consumption was 89\% lower than the cloud-based baseline and throughput scaled linearly with agent count up to 10{,}000 devices. Response time on the evaluated testbed was roughly 1{,}000$\times$ lower than the cloud-based baseline, a comparison that reflects the specific testbed conditions described in Section~\ref{sec:methodology}.

DMAS moves IIoT security away from the centralized, reactive architectures that have dominated industrial deployments toward a fully distributed model where edge agents collaboratively form a real-time defense layer. The combination of domain-aware threat scoring, distance-weighted reputation voting, and Byzantine-tolerant consensus in a single lightweight architecture addresses the latency, scalability, and fault tolerance requirements that centralized approaches cannot jointly satisfy. As IIoT deployments grow in scale and criticality, distributing security intelligence to the network edge will be increasingly necessary to meet the response-time and reliability demands of industrial control.

Future work will integrate federated learning for distributed model improvement, explainable AI for operator-facing threat interpretation, and cross-domain intelligence sharing to extend DMAS beyond single-facility deployments.

\balance
\end{document}